\documentclass{article}

\usepackage{fullpage}
\usepackage{hyperref,url}
\usepackage{amsmath,amsfonts,amsthm,amssymb}
\usepackage{bbm}
\usepackage{graphics, graphicx, xcolor}
\usepackage{enumerate}
\usepackage{enumitem}
\usepackage{verbatim}		% for misc commenting, etc.
\usepackage{psfrag}
\usepackage[font={small}, margin=1cm]{caption}
\usepackage{subcaption}

\newtheorem{thm}{Theorem}%[section]
\newtheorem{lem}[thm]{Lemma}
\newtheorem{prop}[thm]{Proposition}

\newtheoremstyle{named}{}{}{}{}{\bfseries}{.}{.5em}{\thmnote{#3}}
\theoremstyle{named}
\newtheorem{nremark}{Remark}

\newcommand{\ones}[1]{\mathbbm{1}^{#1}}
\newcommand{\va}{\mathbf{a}}
\newcommand{\vF}{\mathbf{F}}    % constraints defined by function
\newcommand{\vu}{\mathbf{u}}    % upper bound on correlations
\newcommand{\vg}{\mathbf{g}}    % algorithm's predictions
\newcommand{\vp}{\mathbf{p}}
\newcommand{\vq}{\mathbf{q}}
\newcommand{\vr}{\mathbf{r}}

\newcommand{\vv}{\mathbf{v}}
\newcommand{\vz}{\mathbf{z}}
\newcommand{\vzero}{\mathbf{0}}

\DeclareMathOperator*{\argmin}{arg\,min}
\DeclareMathOperator*{\argmax}{arg\,max}
\DeclareMathOperator{\sgn}{sgn}
\DeclareMathOperator{\Prtxt}{Pr}

\newcommand{\RR}{\mathbb{R}}      % Real numbers
\newcommand{\ifn}{\mathbf{1}} % indicator function for sets
\newcommand{\evp}[2]{\mathbb{E}_{#2} \left[#1\right]} % expected value operator
\newcommand{\abs}[1]{\left| #1 \right|}

\newcommand{\prp}[2]{\Prtxt_{#2} \left(#1\right)}

\newcommand{\err}[1]{\mbox{err}\left(#1\right)}
\newcommand{\emperr}[2]{\widehat{\mbox{err}}_{#2} \left(#1\right)}

\newcommand*{\qedinpw}{\hfill\ensuremath{\square}} % white box

\newcommand{\expp}[1]{\exp \left(#1\right)}

\newcommand{\cH}{\mathcal{H}}
\newcommand{\cX}{\mathcal{X}}
\newcommand{\cY}{\mathcal{Y}}

\newcommand{\cD}{\mathcal{D}}

\newcommand{\lrp}[1]{\left(#1\right)}

\begin{document}

\title{PAC-Bayes with Minimax for Confidence-Rated Transduction}

\author{
Akshay Balsubramani \\
\texttt{abalsubr@cs.ucsd.edu} \\
University of California, San Diego
\and
Yoav Freund \\
\texttt{yfreund@cs.ucsd.edu} \\
University of California, San Diego
}
\date{}

\maketitle

% Assume: current p_t has high entropy, target p_{t+1} also does. Therefore if p_t doesn't weight h too low when h is weighted highly by p_{t+1}, we're in business. 
% Now Z_{t+1} < Z_t, and then expected empirical error of Z_t comes into play. Also must account for partitioning of hypotheses into labelings, because all entropies are even higher conditioned on this partitioning.

\begin{abstract} 
We consider using an ensemble of binary classifiers for transductive prediction, 
when unlabeled test data are known in advance.
We derive minimax optimal rules for confidence-rated prediction in this setting.
By using PAC-Bayes analysis on these rules, 
we obtain data-dependent performance guarantees without distributional assumptions on the data. 
Our analysis techniques are readily extended to a setting in which the predictor is allowed to abstain.
\end{abstract}

\section{Introduction}
Modern applications of binary classification have recently driven renewed theoretical interest 
in the problem of confidence-rated prediction \cite{EYW10, WY11, NCGCVVF11}.
This concerns classifiers which, for any unlabeled data example, 
output an encoding of the classifier's confidence in its own label prediction. 
Confidence values can subsequently be useful for active learning or further post-processing.

Approaches which poll the predictions of many classifiers in an ensemble $\cH$ 
are of particular interest for this problem \cite{FMS04, BLRR04}. 
The Gibbs (averaged) classifier chooses a random rule from the ensemble and predicts with that rule. 
Equivalently, one can say that the prediction on a particular unlabeled example  
is randomly $+1$ or $-1$ with probabilities proportional to the number of votes garnered by the corresponding labels. 
This is intuitively appealing, 
but it ignores an important piece of information - 
the average error of the Gibbs predictor, which we denote by $\lambda$. 
If the ratio between the $+1$ and $-1$ votes is more extreme than $\lambda$, 
then the intuition is that the algorithm should be fairly confident in the majority prediction. 
The main result of this paper is a proof that
a slight variation of this rough argument holds true, 
suggesting ways to aggregate the classifiers in $\cH$ when the Gibbs predictor is not optimal.

\iffalse
The properties of ensemble predictions are often studied with traditional PAC-Bayes analysis, 
which we do use here, though in a somewhat atypical spirit. 
The PAC-Bayes theorem focuses on the average (Gibbs) empirical error
of a weighted ensemble $\cH$.
In this paper, we use minimax analysis in a transductive setting 
to convert such bounds on average classifier performance  
into better guarantees for a prediction rule 
that exhibits true \emph{voting} behavior to aggregate the classifiers in the ensemble.
\fi

We consider a simple transductive prediction model in which the label 
predictor and nature are seen as opponents playing a zero-sum game. 
In this game, the predictor chooses a prediction $g_i \in [-1,1]$ on the $i^{th}$ unlabeled example, 
and nature chooses a label $z_i \in [-1,1]$. 
\footnote{\label{ftnote:stochlabels} This can be thought of as parametrizing a stochastic binary label; 
for instance, $z_i = -0.5$ would be equivalent to choosing the labels $(-1,1)$ 
with respective probabilities $(0.75, 0.25)$.}
The goal of the predictor is to maximize the average correlation $\evp{g_i z_i}{i} = \frac{1}{n} \sum_{i=1}^n g_i z_i$ 
over $n$ unlabeled examples, 
while nature plays to minimize this correlation. 

Without additional constraints, nature could use the trivial strategy of always choosing $z_i = 0$, 
in which case the correlation would be zero regardless of the choices made by the predictor. 
Therefore, we make one assumption -  
that the predictor has access to an ensemble of classifiers which on average have small error. 
Clearly, under this condition nature cannot use the trivial strategy. 
The central question is then: What is the optimal way for the predictor
to combine the predictions of the ensemble?
That question motivates the main contributions of this paper:
\begin{itemize}[leftmargin=*]
\item 
Identifying the minimax optimal strategies for the predictor and for nature, and the resulting minimax value of the game.
\item 
Applying the minimax analysis to the PAC-Bayesian framework to derive PAC-style guarantees. 
We show that the minimax predictor cannot do worse than the average ensemble prediction, 
and quantify situations in which it enjoys better performance guarantees.
\item 
Extending the analysis to the case in which the predictor can \emph{abstain} from committing to any label and instead suffer a fixed loss. 
A straightforward modification of the earlier minimax analysis expands on prior work in this setting.
\end{itemize}

\section{Preliminaries}
\label{sec:setup}
The scenario we have outlined can be formalized with the following definitions.
\begin{enumerate}[leftmargin=2.4em]
\item {\bf Classifier ensemble:} 
A finite set of classification rules $\cH = \{h_1,\ldots,h_H \}$ 
that map examples $x \in \cX$ to labels $y \in \cY := \{-1,1\}$. 
This is given to the predictor, with a distribution $\vq$ over $\cH$.

\item {\bf Test set:} $n$ unlabeled examples $x_i \in \cX$, 
denoted $T = \{x_1,\ldots,x_n\}$. 
\footnote{We could also more generally assume we are given a distribution $\vr$ over $T$ 
which unequally weights the points in $T$. 
The arguments used in the analysis remain unchanged in that case. }

\item {\bf Nature:} 
Nature chooses a vector $\vz \in [-1,1]^n$ encoding the label associated with each test example.
This information is unknown to the predictor. 

\item {\bf Low average error:} 
Recall the distribution $\vq$ over the ensemble given to the predictor. 
We assume that for some $\lambda > 0$,
$\displaystyle \frac{1}{n} \sum_{i=1}^n \sum_{j=1}^H q_j h_j(x_i) z_i = \lambda$.
So the average correlation between the prediction of a randomly chosen classifier from the ensemble
and the true label of a random example from $T$ is at least $\lambda$. 
\footnote{Equivalent to the average classification error 
%$\frac{1}{n} \sum_{i=1}^n \lrp{\frac{1+a_i}{2} \lrp{1 - \frac{1+z_i}{2}} + \frac{1+z_i}{2} \lrp{1 - \frac{1+a_i}{2}}} = \frac{1}{2} \lrp{1 - \frac{1}{n} \vz^\top \va}$ 
being $\leq \frac{1}{2} (1 - \lambda)$.}

\item{\bf Notation:} 
For convenience, we denote by $\vF$ the matrix that contains the predictions of
$(h_1,\ldots,h_H)$ on the examples $(x_1,\ldots,x_n)$. 
$\vF$ is independent of the true labels, and is fully known to the predictor.
\begin{equation}
\vF = 
 \begin{pmatrix}
   h_1(x_1) & h_2(x_1) & \cdots & h_H (x_1) \\
   h_1(x_2) &  h_2(x_2) & \cdots & h_H (x_2) \\
   \vdots   & \vdots    & \ddots &  \vdots  \\
   h_1(x_n)  &  h_2(x_n)  & \cdots &  h_H (x_n)
 \end{pmatrix}
\end{equation}
In this notation the bound on the average error is expressed as:
$\frac{1}{n} \vz^\top \vF \vq \geq \lambda$.

\end{enumerate}

%----------------------------------------------------------------------------------------------------------------------------------------------------------------------------------------------------------------------------------
%----------------------------------------------------------------------------------------------------------------------------------------------------------------------------------------------------------------------------------

\section{The Confidence-Rated Prediction Game}
\label{sec:game1}

In this game, the goal of the predictor is to find a function $g : T \mapsto [-1,1]^n$ (a vector in $ \RR^n $), 
so that each example $x_i$ maps to a confidence-rated label prediction $g_i \in [-1,1]$. 
The predictor maximizes the worst-case correlation with the true labels $\vz^\top \vg$, 
predicting with the solution $\vg^*$ to the following game:
\begin{eqnarray}
\label{game1eq}
\mbox{Find: }&\displaystyle \max_{\vg} \min_{\vz} \;\; \frac{1}{n} \vz^\top \vg \;\; \\
\mbox{Such that: }& \frac{1}{n} \vz^\top \va \geq \lambda \mbox{ and } -\ones{n} \leq \vz \leq \ones{n} , \nonumber \\
& -\ones{n} \leq \vg \leq \ones{n} \notag
\end{eqnarray}
where $\va := \vF \vq \in \RR^n$ represents the ensemble predictions on the dataset:
\begin{align*}
\va =
 \begin{pmatrix}
   \displaystyle
   \sum_{j=1}^H q_j h_j(x_1) \;,\; \sum_{j=1}^H q_j h_j(x_2) ,\dots, \sum_{j=1}^H q_j h_j(x_n)
 \end{pmatrix}^\top
\end{align*}

It is immediate from the formulation \eqref{game1eq} that the predictor can simply play $\vg = \va$, 
which will guarantee correlation $\lambda$ due to the average performance constraint.
This is the prediction made by the Gibbs classifier $h_\vq (x) = \evp{h(x)}{h \sim \vq}$, 
which averages across the ensemble. 
We now identify the optimal strategy for the predictor and show
when and by how much it outperforms $h_\vq$.
\footnote{\label{ftnote:identensemble}
If every $h \in \cH$ makes identical predictions on the dataset 
and has correlation $\lambda$ with the true labels, 
then the ensemble effectively has just one element, 
so outperforming it is impossible without outside information.}

%----------------------------------------------------------------------------------------------------------------------------------------------------------------------------------------------------------------------------------
%----------------------------------------------------------------------------------------------------------------------------------------------------------------------------------------------------------------------------------

\subsection{Analytical Solution}
\label{sec:game1analysis}

Consider the game as defined in \eqref{game1eq}. 
The conditions for minimax duality hold here (Prop. \ref{prop:game1duality} in the appendices for completeness), 
so the minimax dual game is
\begin{eqnarray}
\displaystyle \min_{\substack{ \vz \in [-1,1]^n , \\ \frac{1}{n} \vz^\top \va \geq \lambda }} \max_{\vg \in [-1,1]^n} \;\; \frac{1}{n} \vz^\top \vg  \label{game1dual}
\end{eqnarray}

Without loss of generality, we can reorder the examples so that the
following condition holds on the vector of ensemble predictions $\va$.
\begin{nremark}[Ordering \ref{remark:regorder}]
\label{remark:regorder}
Order the examples so that $\displaystyle \abs{a_1} \geq \abs{a_2} \geq \dots \geq \abs{a_n}$. 
\qedinpw
\end{nremark}

Our first result expresses the minimax value of this game. 
(All proofs are deferred to the appendices.)

\begin{lem}
\label{lem:game1val}
Using Ordering \ref{remark:regorder} of the examples, 
let $v = \min \left\{ i \in [n] : \frac{1}{n} \sum_{j=1}^i \abs{a_j} \geq \lambda \right\}$. 
Then the value of the game \eqref{game1eq} is 
$\displaystyle V := \frac{v-1}{n} + \frac{1}{\abs{a_v}} \lrp{ \lambda - \frac{1}{n} \sum_{i=1}^{v-1} \abs{a_i} }$.
\end{lem}

This allows us to verify the minimax optimal strategies.
\begin{thm}
\label{thm:game1soln}
Suppose the examples are in Ordering \ref{remark:regorder}, 
and let $v$ be as defined in Lemma \ref{lem:game1val}. 
The minimax optimal strategies for the predictor ($\vg^*$) and nature ($\vz^*$) 
in the game \eqref{game1eq} are:
\begin{align*} 
&g_i^* = \begin{cases} \sgn(a_i) & i \leq v \\ \frac{a_i}{\abs{a_v}} & i > v \end{cases} \\ 
&z_i^* = \begin{cases} \sgn(a_i) & i < v \\ \frac{1}{a_i} \lrp{ n \lambda - \sum_{i=1}^{v-1} \abs{a_i} } & 
i = v \\ 0 & i > v \end{cases} 
\end{align*}
\end{thm}
%The components of $\vg^*, \vz^*$ are plotted against the respective $a_i$ values in Figure \ref{fig:minipage1} 
%(for a very large $n$; when $\abs{a_i} > \abs{a_v}$, the two lines are concurrent).

\begin{figure}[b]
\centering
\begin{minipage}[b]{0.45\linewidth}
\centering
\includegraphics[height=1.2in, width=0.8\textwidth]{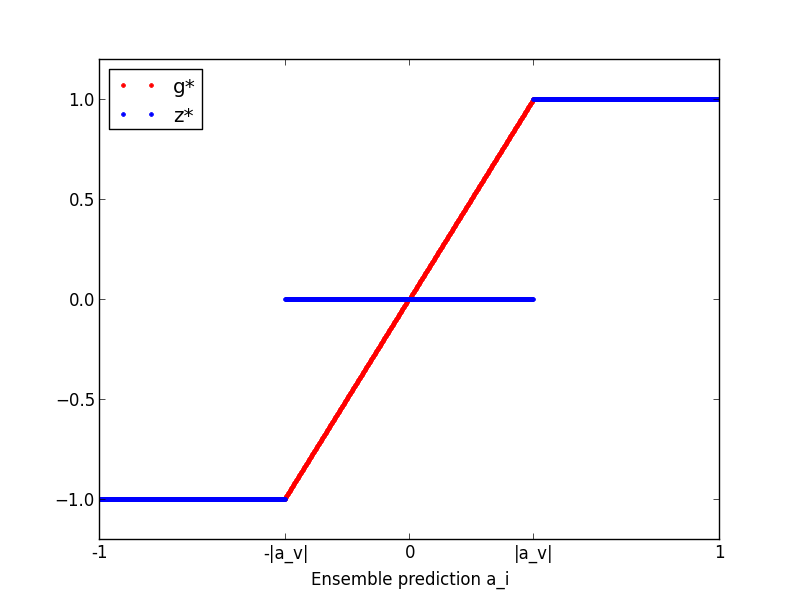}
\caption{Optimal strategies $\vg^*, \vz^*$ for the game without abstention 
(Thm. \ref{thm:game1soln}), plotted against $a_i$. }
\label{fig:minipage1}
\end{minipage}
\quad
\begin{minipage}[b]{0.45\linewidth}
\centering
\includegraphics[height=1.2in, width=0.8\textwidth]{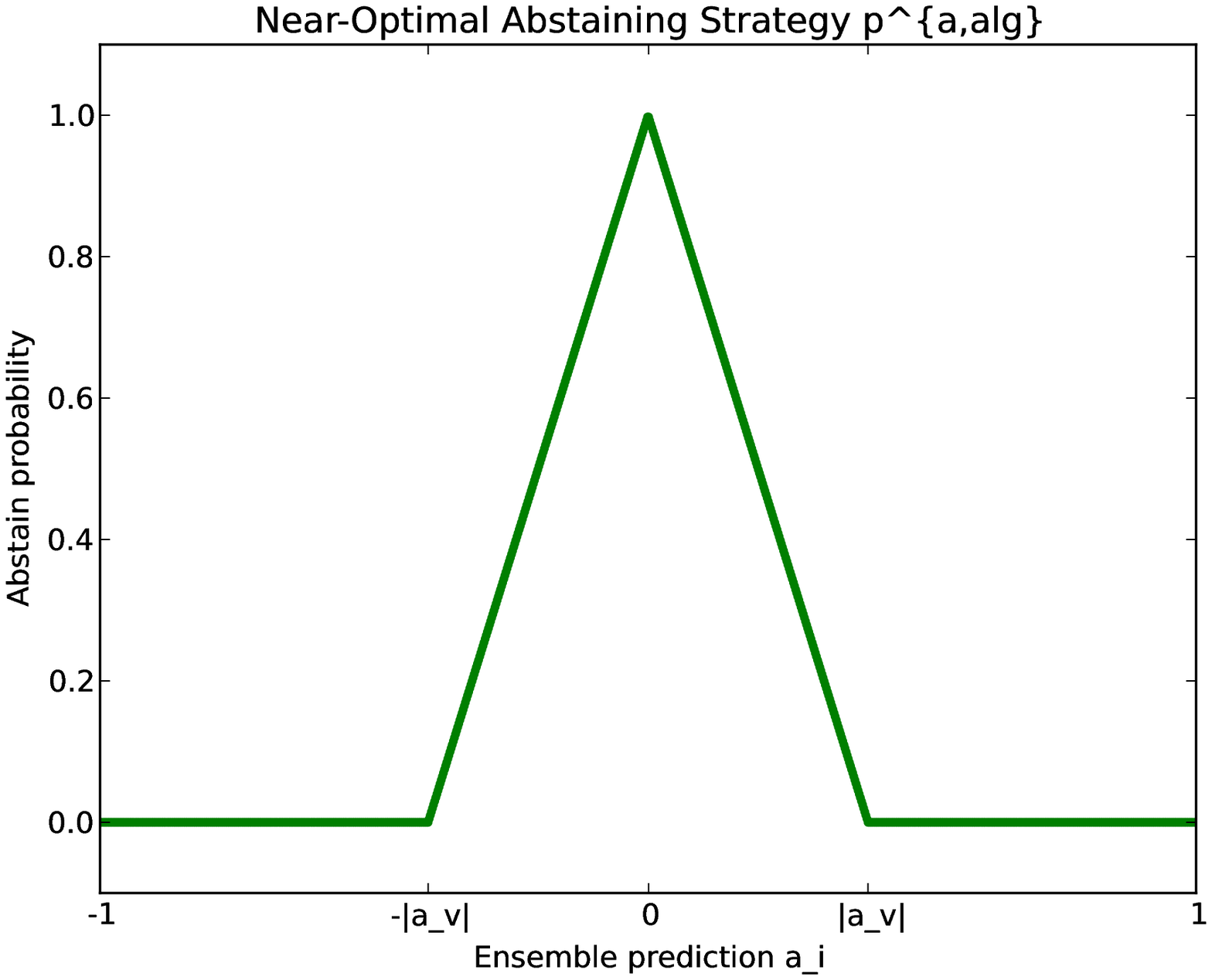}
\caption{Near-optimal abstain probabilities $\vp^{a,alg}$ (Thm. \ref{thm:gameabsapprox}), plotted against $a_i$.}
\label{fig:minipage2}
\end{minipage}
\end{figure}

%\begin{figure}[h]
%\centering
%\includegraphics[height=1.2in, width=0.8\linewidth]{figures/mmxoptstrat.eps}
%\caption{Optimal strategies $\vg^*, \vz^*$ for the game without abstention 
%(Thm. \ref{thm:game1soln}), plotted against $a_i$. }
%\label{fig:noabsplot}
%\end{figure}

\subsection{Discussion}
\label{sec:1gamediscuss}

As shown in Fig. \ref{fig:noabsplot}, the optimal strategy $\vg^*$ depends on the Gibbs classifier's prediction $\va$ in an elegant way. 
For the $V$ fraction of the points on which the ensemble is most certain (indices $i < v$), 
the minimax optimal prediction is the deterministic majority vote. 

For any example $i$, 
$g_i^*$ is a nondecreasing function of $a_i$, 
which is often an assumption on $\vg$ in similar contexts \cite{ZE02}. 
In our case, this monotonic behavior of $\vg^*$ arises \emph{entirely} from nature's average error constraint; 
$\vg$ itself is only constrained to be in a hypercube.

When $\lambda \approx 0$ ($v \approx 1$), $\vg^*$ approximates the ensemble prediction.  
But as the ensemble's average correlation $\lambda$ increases, 
the predictor is able to act with certainty ($\abs{g_i^*} = 1$) 
on a growing number of examples, 
even when the vote is uncertain ($\abs{a_i} < 1$). 

The value of the game can be written as 
\begin{align}
\label{valtbound}
V &= \frac{v-1}{n} + \frac{1}{\abs{a_v}} \lrp{ \lambda - \frac{1}{n} \sum_{i=1}^{v-1} \abs{a_i} } \nonumber \\
&\geq \frac{v-1}{n} + \lambda - \frac{1}{n} \sum_{i=1}^{v-1} \abs{a_i}
= \lambda + \frac{1}{n} \sum_{i=1}^{v-1} \lrp{1 - \abs{a_i}}
\end{align}
This shows that predicting with $\vg^*$ cannot hurt performance relative to the average ensemble prediction, 
and indeed will help when there are disagreements in the ensemble on high-margin examples. 
The difference $V - \lambda = \frac{1}{n} \sum_{i=1}^{v-1} \lrp{1 - \abs{a_i}}$ 
quantifies the benefit of our prediction rule's voting rule as opposed to the Gibbs classifier's averaged prediction. 

Our minimax analysis is able to capture this uniquely vote-based behavior because of the transductive setting. 
The relationships between the classifiers in the ensemble determine the performance of the vote, 
and analyzing such relationships is much easier in the transductive setting. 
There is a dearth of applications of this insight in existing literature on confidence-rated prediction, 
with \cite{SV08} being a notable exception.

In this work, the predictor obtains a crude knowledge of the hypothesis predictions $\vF$ 
through ensemble predictions $\va$, and is thereby able to quantify the benefit of a vote-based predictor in terms of $\va$. 
However, the loss of information in compressing $\vF$ into $\va$ is unnecessary, 
as $\vF$ is fully known to the predictor in a transductive setting. 
It would therefore be interesting (and provide a tighter analysis) 
to further incorporate the structure of $\vF$ into the prediction game in future work.

\begin{figure}[t]
\centering

\includegraphics[height=1.5in, width=0.6\linewidth]{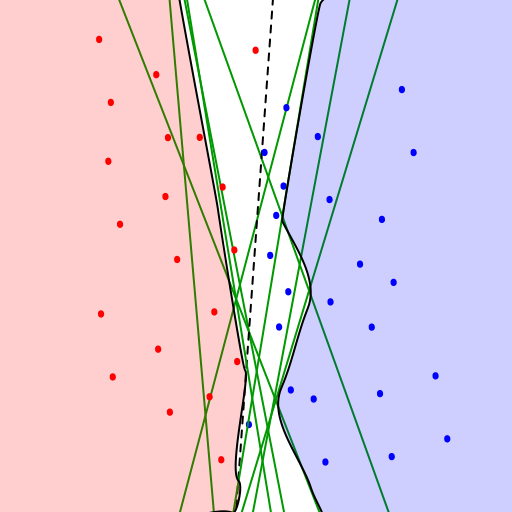}
\caption{An illustration of the minimax optimal predictor $g^*$, 
with the linear separators being low-error hypotheses in $\cH$, 
and the examples colored according to their true labels. 
The red and blue shaded areas indicate where the predictor is certain. }
%
%\begin{minipage}[b]{0.45\linewidth}
%\centering
%\includegraphics[height=1.2in, width=0.8\textwidth]{figures/abststrat}
%\caption{Near-optimal abstain probabilities $\vp^{a,alg}$ (Thm. \ref{thm:gameabsapprox}), plotted against $a_i$.}
%\label{fig:minipage2}
%\end{minipage}
\end{figure}

%----------------------------------------------------------------------------------------------------------------------------------------------------------------------------------------------------------------------------------
%----------------------------------------------------------------------------------------------------------------------------------------------------------------------------------------------------------------------------------

%----------------------------------------------------------------------------------------------------------------------------------------------------------------------------------------------------------------------------------
%----------------------------------------------------------------------------------------------------------------------------------------------------------------------------------------------------------------------------------

\section{A PAC-Bayes Analysis of a Transductive Prediction Rule}
\label{sec:pbanalysis}

The minimax predictor we have described relies on a known average correlation 
$\lambda > 0$ between the ensemble predictions and true labels. 
In this section, 
we consider a simple transductive statistical learning scenario 
in which the ensemble distribution $\vq$ is learned from a training set 
using a PAC-Bayesian criterion, 
giving a statistical learning algorithm with a PAC-style analysis. 

Suppose we are in a transductive setting with a training set $S$ with known labels, 
and a test set $T$ with unknown labels. 
$S$ and $T$ are assumed to be composed of labeled examples drawn i.i.d. from a distribution $\cD$ over $\cX \times \cY$. 
We write $\abs{S} = m$ and consider $\abs{T} > m$. 

Denote the true error of a hypothesis $h \in \cH$ by $\err{h} = \prp{h(x) \neq y}{\cD}$, 
its empirical error on $S$ by $\emperr{h}{S} = \frac{1}{m} \sum_{(x,y) \in S} \ifn(h(x) \neq y)$, 
and its empirical error on $T$ by $\emperr{h}{T}$. 
Also, for any $p,q \in [0,1]$ define $KL\lrp{ p \mid \mid q } = p \log \frac{p}{q} + (1-p) \log \frac{1-p}{1-q}$, 
and otherwise define the KL divergence $KL\lrp{ \vp \mid \mid \vq }$ between two distributions $\vp, \vq \in \RR^n$ in the usual way as
$\sum_{i=1}^n p_i \log \frac{p_i}{q_i}$.

Finally, define $ \epsilon (m, \vq, \vq_0, \delta) := \sqrt{ \frac{2}{m} \lrp{ KL (\vq \mid \mid \vq_0) + \log\lrp{\frac{2(m+1)}{\delta}} } }$ 
and the error-minimizing hypothesis $h^* = \argmin_{h \in \cH} \err{h}$.

\subsection{An Algorithm with a PAC-Bayes Analysis}
\label{sec:pbalgbasic}

\textbf{Algorithm:  }The learning algorithm we consider simply observes the labeled set $S$ 
and chooses a distribution $\vq$ over $\cH$ that has low error on $S$. 
Based on this, it calculates a lower bound $\hat{\lambda}$ (Eq. \eqref{1phaselambda})
 on the correlation with $T$ of the associated Gibbs classifier.  
Finally, it uses $\hat{\lambda}$ in the average error constraint for the game of Section \ref{sec:game1}, 
and predicts with the corresponding minimax optimal strategy $\vg^*$ as given in Theorem \ref{thm:game1soln}.

\textbf{Analysis:  }We begin analyzing this algorithm by applying the PAC-Bayes theorem (\cite{L05}) 
to control $\evp{\emperr{h}{S}}{h \sim \vq}$, which immediately yields the following.

\begin{lem}
\label{lem:1phasepbonS}
Choose any prior distribution $\vq_0$ over $\cH$. 
With probability $\geq 1 - \delta$ over the choice of training set $S$, for all distributions $\vq$ over $\cH$ simultaneously,
\begin{align}
KL &\lrp{ \evp{\emperr{h}{S}}{h \sim \vq} \mid \mid \evp{\err{h}}{h \sim \vq} } \\ 
&\leq \frac{1}{m} \lrp{ KL (\vq \mid \mid \vq_0) + \log\lrp{\frac{m+1}{\delta}} } \nonumber
\end{align}
\end{lem}

This can be easily converted into a bound on $\emperr{h}{T}$, 
using a Hoeffding bound and the well-known inequality $KL\lrp{ p \mid \mid q } \geq 2 (p-q)^2$:
\begin{thm}
\label{thm:1phaseStoT}
Choose any prior distribution $\vq_0$ over the hypotheses $\cH$. 
With probability $\geq 1 - \delta$, for all distributions $\vq$ over $\cH$ simultaneously,
\begin{align}
\evp{\emperr{h}{T}}{h \sim \vq} \leq \evp{\emperr{h}{S}}{h \sim \vq} + \epsilon (m, \vq, \vq_0, \delta) \nonumber
\end{align}
\end{thm}
%\begin{proof}
%It is a well-known inequality that for scalars $p, q \in [0,1]$, $KL\lrp{ p \mid \mid q } \geq 2 (p-q)^2$. 
%Substituting this into Lemma \ref{lem:1phasepbonS}, we get that w.p. $\geq 1 - \frac{\delta}{2}$,
%\begin{align}
%\label{1phaseStoG}
%\evp{\err{h}}{h \sim \vq} \leq \evp{\emperr{h}{S}}{h \sim \vq} + \sqrt{ \frac{1}{2 m} \lrp{ KL (\vq \mid \mid \vq_0) + \log\lrp{\frac{2 (m+1)}{\delta}} } }
%\end{align}
%Now a Hoeffding bound gives that w.p. $\geq 1 - \frac{\delta}{2}$, 
%we have $\evp{\emperr{h}{T}}{h \sim \vq} \leq \evp{\err{h}}{h \sim \vq} + \sqrt{ \frac{1}{2 \abs{T}} \log\lrp{\frac{2}{\delta}} } $.
%Taking a union bound with \eqref{1phaseStoG} and combining the two gives the result.
%\end{proof}

Theorem \ref{thm:1phaseStoT} w.h.p. controls the average classifier performance on the test set. 
Recall that this constrains nature in the minimax analysis, where it is expressed as a lower bound $\hat{\lambda}$ 
on the average correlation $1 - 2 \evp{\emperr{h}{T}}{\vq}$.
From Theorem \ref{thm:1phaseStoT}, with probability $\geq 1 - \delta$, for all $\vq$,
\begin{align}
\label{1phaselambda}
\hat{\lambda} = 1 - 2 \evp{\emperr{h}{S}}{\vq} - 2 \epsilon (m, \vq, \vq_0, \delta)
\end{align}

Inside this high-probability event, the scenario of Section \ref{sec:game1} (the prediction game) holds, 
with $\lambda$ given by \eqref{1phaselambda}. 
A higher correlation bound $\hat{\lambda}$ leads to better performance. 

In the game of Section \ref{sec:game1}, values in $[-1,1]$ can be thought of as parametrizing stochastic 
binary labels/predictions (see Footnote \ref{ftnote:stochlabels}).
By this token, if the value of the game is $V$, 
the prediction algorithm is incorrect with probability at most 
$\frac{1}{2} \lrp{1 - V} \leq \frac{1}{2} \lrp{1 - \hat{\lambda} - \frac{1}{n} \sum_{i=1}^{v-1} \lrp{1 - \abs{a_i}} }$
(from \eqref{valtbound}). 
Combining this with \eqref{1phaselambda} and a union bound, the algorithm's probability of error is at most 
\begin{align}
\label{errprobalg}
\evp{\emperr{h}{S}}{\vq} - \frac{1}{2n} \sum_{i=1}^{v-1} \lrp{1 - \abs{a_i}} + \epsilon (m, \vq, \vq_0, \delta) + \delta
\end{align}
Since \eqref{1phaselambda} holds uniformly over $\vq$, 
the training set $S$ can be used to set $\vq$ to minimize $\evp{\emperr{h}{S}}{\vq}$. 
However, naively choosing a point distribution on $\argmin_h \emperr{h}{S}$ 
(performing empirical risk minimization) is inadvisable, 
because it does not hedge against the sample randomness. 
In technical terms, a point distribution leads to a high KL divergence term (using a uniform prior $\vq_0$) in \eqref{1phaselambda}, 
and of course eliminates any potential benefit from voting with respect to the Gibbs classifier.
Instead, a higher-entropy distribution is apropos, 
like the exponential-weights distribution $q_{exp} (h) \propto \expp{-\eta \emperr{h}{S} }$.

Regardless, for sufficiently high $m$, 
$\epsilon (m, \delta)$ is negligible and we can take $\evp{\emperr{h}{S}}{\vq} \approx \err{h^*}$.  
So in this regime the classification error is $\lesssim \err{h^*}$; 
it can be much lower due to the final term of \eqref{errprobalg}, 
again highlighting the benefit of voting behavior when there is ensemble disagreement. 

%There is a long history of use of this and similar distributions in PAC-Bayes analysis of the Gibbs error.  

\begin{nremark}[Many Good Classifiers and Infinite $\cH$]
The guarantees of this section are particularly nontrivial if there is a significant ($\vq$-)fraction of hypotheses in $\cH$ with low error. 
The uniform distribution over these good hypotheses (call it $\vu_g$) has support of a significant size, 
and therefore $KL \lrp{\vu_g \mid \mid \vq_0}$ is low, where $\vq_0$ is a uniform prior over $\cH$. 
(The same approach extends to infinite hypothesis classes.)
%\qedinpw
\end{nremark}

\subsection{Discussion}
%We choose a confidence-rated prediction that minimizes a loss that treats the confidence as a label probability. 
%There are many other approaches to the problem \cite{SV08}

Our approach to generating confidence-rated predictions has two distinct parts: 
the PAC-Bayes analysis and then the minimax analysis. 
This explicitly decouples the selection of $\vq$ from the aggregation of ensemble predictions, 
and makes clear the sources of robustness here:
PAC-Bayes works with any data distribution and any posterior $\vq$, 
and the minimax scenario works with any $\vq$ and yields worst-case guarantees.
It also means that either of the parts is substitutable.

The PAC-Bayes theorem itself admits improvements 
which would tighten the results achieved by our approach. 
Two notable ones are the use of generic chaining to incorporate finer-grained complexity of $\cH$ \cite{AB07}, 
and an extension to the transductive online setting of sampling without replacement, 
in which no i.i.d. generative assumption is required \cite{BGLR14}.

%The results in this case can be summarized as follows (calculations in Appendix \ref{sec:pbcalc}).
%\begin{enumerate}
%\item
%The algorithm predicts ($\neq \perp$) incorrectly with probability at most 
%\begin{align}
%\label{absterrprob}
%\evp{\emperr{h}{S}}{\vq} + \epsilon (m, \delta) - \frac{1}{2} \sum_{i=1}^v r_i \lrp{1 - \abs{a_i}}
%\end{align}
%\item
%When the algorithm abstains at all ($\alpha \leq \frac{1}{2}$), it does so with probability at most 
%\begin{align}
%\label{abstprob}
%2 \evp{\emperr{h}{S}}{\vq} + 2 \epsilon (m, \delta) - \sum_{i > v} r_i \frac{\abs{a_i}}{\abs{a_v}}
%\end{align}
%\end{enumerate}

%----

%----------------------------------------------------------------------------------------------------------------------------------------------------------------------------------------------------------------------------------
%----------------------------------------------------------------------------------------------------------------------------------------------------------------------------------------------------------------------------------
%----------------------------------------------------------------------------------------------------------------------------------------------------------------------------------------------------------------------------------
%----------------------------------------------------------------------------------------------------------------------------------------------------------------------------------------------------------------------------------

\section{Extension to Abstention}
\label{sec:abstaingame}
This section outlines a natural extension of the previous binary classification game, 
in which the predictor can choose to abstain from committing to a label and suffer a fixed loss instead.
We model the impact of abstaining by treating it as a third classification outcome 
with a relative cost of $\alpha > 0$, 
where $\alpha$ is a constant independent of the true labels. 
Concretely, consider the following general modification of our earlier game with parameter $\alpha > 0$:

\begin{enumerate}[leftmargin=2.4em]
\item {\bf Predictor:} 
On an example $i \in [n]$, the predictor can either predict a value in $[-1,1]$ or abstain (denoted by an output of $\perp$). 
When it does not abstain, it predicts $g_i \in [-1,1]$ as in the previous game.
But it does so only with probability $1 - p_i^a$ for some $p_i^a \in [0,1]$; 
the rest of the time it abstains, 
where $\prp{\mbox{output $\perp$ on } i}{} = 1 - \prp{\mbox{predict $g_i$ on } i}{} = p_i^a$.
So the predictor's strategy in the game is a choice of $(\vg, \vp^a)$, 
where $\vg = (g_1, \dots, g_n)^\top \in [-1,1]^n$ and $\vp^a = (p_1^a, \dots, p_n^a)^\top \in [0,1]^n$.

\item {\bf Nature:} 
Nature chooses $\vz$ as before to represent its randomized label choices for the data.

\item {\bf Cost model:} 
The predictor suffers cost (nature's gain) of the form $\frac{1}{n} \sum_{i=1}^n l_i (\hat{z}_i, z_i)$, 
where $\hat{z}_i \in \{g_i, \perp\} $ is the predictor's output. 
The cost function $l_i (\cdot, \cdot)$ incorporates abstention using a constant loss $\alpha > 0$ 
if $\hat{z}_i = \perp$, regardless of $z_i$:
%{\color{red}Fix this cost function, because it's only for binary predictions, need to make it an expected-value thing.}
\begin{align}
\label{costfnabst}
l_i (\hat{z}_i, z_i) = \begin{cases} \frac{1}{2} (1 - g_i z_i) &\;\; \hat{z}_i = g_i \\
\alpha &\;\; \hat{z}_i = \perp \end{cases}
\end{align}
\end{enumerate}

In this game (the ``abstain game"), 
the predictor wishes to minimize the expected loss w.r.t. the stochastic strategies of itself and nature, 
and nature plays to maximize this loss.
So the game can be formulated as:
\begin{align}
\label{abstaingame}
%\min_{\vp^a \in [0,1]^n} &\min_{\vg \in [-1,1]^n} \max_{\substack{ -\vr \leq \vt \leq \vr, \\ \vt^\top \va \geq \lambda}}\; 
%\sum_{i=1}^n \left[ r_i p_i^a \alpha + r_i \lrp{1 - p_i^a} \lrp{ \lrp{\frac{1+g_i}{2}} \lrp{1 - \frac{1+z_i}{2}} + \lrp{\frac{1+z_i}{2}} \lrp{1 - \frac{1+g_i}{2}}} \right] \nonumber \\
%&= 
&\min_{\substack{ \vp^a \in [0,1]^n , \\ \vg \in [-1,1]^n}} \max_{\substack{ \vz \in [-1,1]^n , \\ \frac{1}{n} \vz^\top \va \geq \lambda }}\; 
\frac{1}{n} \sum_{i=1}^n \left[ p_i^a \alpha + \frac{1}{2} \lrp{1 - p_i^a} \lrp{ 1 - g_i z_i } \right] \nonumber \\
&= \frac{1}{2} + \frac{1}{n} \Bigg( \min_{\vp^a \in [0,1]^n} \Bigg[ \sum_{i=1}^n \lrp{ \alpha - \frac{1}{2}} p_i^a \nonumber \\ 
&\;\;\;- \frac{1}{2} \max_{\vg \in [-1,1]^n} \min_{\substack{ \vz \in [-1,1]^n , \\ \frac{1}{n} \vz^\top \va \geq \lambda }}\; \sum_{i=1}^n z_i \lrp{1 - p_i^a} g_i \big] \Bigg)
\end{align}

\subsection{Value of the Abstain Game}
\label{sec:abstgameanalysis}
Minimax duality does apply to \eqref{abstaingame}, 
and the dual game is easier to work with. 
Calculating its value leads to the following result.
\begin{thm}
\label{thm:abstgameexact}
The value $V_{abst}$ of the game \eqref{abstaingame} is as follows for $\alpha < \frac{1}{2}$.
\begin{enumerate}
\item
If $\alpha \leq \frac{1}{2} \lrp{1 - \frac{n \lambda}{\sum_{i=1}^n \abs{a_i}}}$, 
then $V_{abst} = \alpha$ (and the game is vacuous with the minimax optimal strategy $\vp^{a*} = \ones{n}$).
\item
If $\alpha > \frac{1}{2} \lrp{1 - \frac{n \lambda}{\sum_{i=1}^n \abs{a_i}}}$, 
then the value is nontrivial and can be bounded. 
Using Ordering \ref{remark:regorder} of the examples, let 
$w = \displaystyle \min \left\{ i \in [n] : \frac{1}{n} \lrp{ \sum_{j=1}^i \abs{a_j} + \sum_{j=i+1}^n (1 - 2 \alpha) \abs{a_j} } \geq \lambda \right\}$. 
Then
$\alpha \lrp{1 - \frac{w}{n}} \leq V_{abst} \leq \alpha \lrp{1 - \frac{w-1}{n}}$.
\end{enumerate}
\end{thm}
The first part of this result implies that if there is a low abstain cost $\alpha > 0$, 
a low enough average correlation $\lambda$, and not many disagreements among the ensemble, 
then it is best to simply abstain a.s. on all points. 
This is intuitively appealing, but it appears to be new to the literature.

Though the value $V_{abst}$ is obtainable as above, 
we find solving for the optimal strategies in closed form to be more challenging. 
As we are primarily interested in the learning problem and therefore tractable strategies for the abstain game, 
we abandon the minimax optimal approach and present a simple near-optimal strategy for the algorithm.
This strategy has several favorable properties 
which facilitate comparison with the rest of the paper and with prior work.

\subsection{Near-Optimal Strategy for the Abstain Game}
\label{sec:heurgameanalysis}

The following is our main result for the abstain game, 
derived in Section \ref{sec:almostabst}.
\begin{thm}
\label{thm:gameabsapprox}
Using Ordering \ref{remark:regorder} of the examples, 
define $v$ as in Lemma \ref{lem:game1val}. 
Let $\vg^*$ be the minimax optimal strategy for the predictor 
in the binary classification game without abstentions, 
as described in Theorem \ref{thm:game1soln}.
Suppose the predictor in the abstain game \eqref{abstaingame} plays $(\vg^*, \vp^{a, alg})$ respectively, where
\begin{align*}
\vp^{a, alg} = \begin{cases} 1 - \abs{\vg^*} & \alpha < \frac{1}{2} \\ \vzero & \alpha \geq \frac{1}{2} \end{cases} 
\end{align*}
\begin{enumerate}
\item
The worst-case loss incurred by the predictor is at most
$\displaystyle \frac{1}{2} \lrp{1 - \frac{v}{n}}$ when $\alpha \geq \frac{1}{2}$, 
and %$ \frac{1}{2} \lrp{1 - \frac{v}{n} } - \sum_{i>v} \lrp{ \frac{1}{2} - \alpha } r_i \lrp{1 - \frac{\abs{a_i}}{\abs{a_v}} }  
$\alpha \lrp{1 - \frac{v}{n}} + \lrp{ \frac{1}{2} - \alpha } \frac{1}{n} \sum_{i>v} \frac{\abs{a_i}}{\abs{a_v}}$ 
when $\alpha < \frac{1}{2}$.
\item
Nature can play $\vz^*$ to induce this worst-case loss.
\end{enumerate} 
\end{thm}

%\begin{figure}[h]
%\centering
%\includegraphics[height=1.2in, width=0.8\linewidth]{figures/abststrat}
%\caption{Near-optimal abstain probabilities $\vp^{a,alg}$ (Thm. \ref{thm:gameabsapprox}), plotted against $a_i$.}
%\end{figure}

We remark that $\vp^{a,alg}$ is nearly minimax optimal in certain situations. 
Specifically, its worst-case loss can be quite close to the ideal $V_{abst} \approx \alpha \lrp{1 - \frac{w}{n}}$, 
which is defined in Theorem \ref{thm:abstgameexact}. 
In fact, $v \geq w$, so for $\vp^{a,alg}$ to be nearly minimax optimal, 
it suffices if the term $\lrp{ \frac{1}{2} - \alpha } \frac{1}{n} \sum_{i>v} \frac{\abs{a_i}}{\abs{a_v}}$ 
is low.
Loosely speaking, this occurs when $\alpha \approx \frac{1}{2}$ or there is much disagreement among the ensemble, 
or when $\lambda$ (therefore $v$) is high. 

The latter is typical when $\err{h^*}$ is low and the PAC-Bayes transduction algorithm is run; 
so in such cases, the simple strategy $\vp^{a,alg}$ is almost minimax optimal. 
Such a low-error case has been analyzed before in the abstaining setting, notably by \cite{EYW11}, 
who extrapolate from the much deeper theoretical understanding of the realizable case (when $\err{h^*} = 0$).

We have argued that $\vp^{a,alg}$ achieves a worst-case loss arbitrarily close to $V_{abst}$ in the limit $\lambda \to 1$, 
regardless of $\alpha$. 
So we too are extrapolating somewhat from the realizable case with the approximately optimal $\vp^{a,alg}$, 
though in a different way from prior work.

\begin{nremark}[Benefit of Abstention]
Theorem \ref{thm:gameabsapprox} clearly illustrates the benefit of abstention when $\vp^{a,alg}$ is played. 
To see this, define $V$ as in \eqref{valtbound} to be the value of the binary classification game without abstention. 
Then the worst-case classification error of the minimax optimal rule without abstention is 
$\frac{1}{2} \lrp{1 - V} \leq \frac{1}{2} \lrp{1 - \frac{v-1}{n}} := L_n$. 
An upper bound on the worst-case loss incurred by the abstaining predictor that plays $(\vg^*, \vp^{a,alg})$ 
is given in Theorem \ref{thm:gameabsapprox}, 
and can be rewritten as 
$L_a := \frac{1}{2} \lrp{1 - \frac{v}{n} } - \frac{1}{n} \sum_{i>v} \lrp{ \frac{1}{2} - \alpha } \lrp{1 - \frac{\abs{a_i}}{\abs{a_v}} } $. 
$L_n - L_a$ is positive for $\alpha < \frac{1}{2}$, 
\footnote{As written here, $L_n - L_a = \frac{1}{n} \sum_{i>v} \lrp{ \frac{1}{2} - \alpha } \lrp{1 - \frac{\abs{a_i}}{\abs{a_v}} } - \frac{1}{2n}$, 
and the $-\frac{1}{2n}$ term can be dispensed with by lowering $L_n$ using a slightly more careful analysis.}
illustrating the benefit of abstention.
There is no benefit if $\alpha \geq \frac{1}{2}$, 
and increasing benefit the further below $\frac{1}{2}$ it gets. 
The $\alpha > \frac{1}{2}$ result is to be expected - 
even the trivial strategy of predicting $\vg = \vzero$ 
is preferable to abstention if $\alpha > \frac{1}{2}$ - 
and echoes long-known results for this cost model in various settings \cite{C57}.
%\qedinpw
\end{nremark}

\begin{nremark}[Cost Model]
The linear cost model we use, with one cost parameter $\alpha$, 
is prevalent in the literature \cite{T00, WY11}, 
as its simplicity allows for tractable optimality analyses in various scenarios \cite{BW08, C57, C70, YW10}. 
Many of these results, however, assume that the conditional label probabilities are known 
or satisfy low-noise conditions \cite{BW08}.
Others explicitly use the value of $\alpha$ \cite{C70}, 
which is an obstacle to practical use because $\alpha$ is often unknown or difficult to compute. 
Our near-optimal prediction rule sidesteps this problem, 
because the strategy $(\vg^*, \vp^{a, alg})$ is independent of $\alpha$ 
in the nontrivial case $\alpha < \frac{1}{2}$. 
To our knowledge, this is unique in the literature, 
and is a major motivation for our choice of $\vp^{a, alg}$.
%\qedinpw
\end{nremark}

\subsection{Guarantees for a Learning Algorithm with Abstentions}
\label{sec:pbabstguarantees}
The near-optimal abstaining rule of Section \ref{sec:heurgameanalysis} can be bootstrapped into 
a PAC-Bayes prediction algorithm that can abstain, 
exactly analogous to Section \ref{sec:pbanalysis} for the non-abstaining case. 
Similarly to that analysis, 
PAC-style results can be stated for the abstaining algorithm for $\alpha < \frac{1}{2}$ 
(calculations in Appendix B), %\ref{sec:pbcalc}), 
using $\vq, \vq_0, \epsilon(\cdot,\cdot,\cdot,\cdot), S$ defined in Section \ref{sec:pbanalysis}, 
and any $\delta \in (0,1)$. 
The algorithm: 
\begin{align*}
&\mbox{abstains w.p. } \\ 
&\leq 2 \evp{\emperr{h}{S}}{\vq} + 2 \epsilon (m, \vq, \vq_0, \delta) + \delta - \frac{1}{n} \sum_{i > v} \frac{\abs{a_i}}{\abs{a_v}} \\ 
&\mbox{and errs } (\neq \perp) \mbox{ w.p. } \\ 
&\leq \evp{\emperr{h}{S}}{\vq} + \epsilon (m, \vq, \vq_0, \delta) + \delta - \frac{1}{2n} \sum_{i=1}^v \lrp{1 - \abs{a_i}} 
\end{align*}
Thus, by the same arguments as in the discussion of Section \ref{sec:pbalgbasic}, 
the abstain and mistake probabilities are respectively 
$\lesssim 2 \err{h^*}$ and $\lesssim \err{h^*}$ for sufficiently large $m$.
Both are sharper than corresponding results of Freund et al. \cite{FMS04}, 
whose work is in a similar spirit. 

Their setup is similar to our ensemble setting for abstention, 
and they choose $\vq$ to be an exponential-weights distribution over $\cH$. 
In their work \cite{FMS04}, 
the decision to abstain on the $i^{th}$ example is a deterministic function of $\abs{a_i}$ only 
(ours is a stochastic function of the full vector $\va$), 
and any non-$\perp$ predictions are made deterministically with the majority vote (ours can be stochastic). 
Predicting stochastically and exploiting transduction lead to our mistake probability being essentially optimal 
(Footnote \ref{ftnote:identensemble}) 
as opposed to the $\approx 2 \err{h^*}$ of Freund et al. \cite{FMS04} caused by averaging effects. 
Our abstain probability also compares favorably to the $\approx 5 \err{h^*}$ in \cite{FMS04}.

\subsection{Derivation of Theorem \ref{thm:gameabsapprox}}
\label{sec:almostabst}

Define another ordering of the examples here: 
\begin{nremark}[Ordering \ref{remark:abstorder}]
\label{remark:abstorder}
Order the examples so that 
$\displaystyle \frac{\abs{a_1}}{1 - p_1^a} \geq \frac{\abs{a_2}}{1 - p_2^a} \geq \dots \geq \frac{\abs{a_n}}{1 - p_n^a}$.
\footnote{
Hereafter we make two assumptions for simplicity. 
One is that there are no examples such that $p_i^a = 1$ exactly. 
The other is to neglect the effect of ties and tiebreaking.
These assumptions do not lose generality for our purposes, 
because coming arbitrarily close to breaking them is acceptable.}
\qedinpw
\end{nremark}

We motivate and analyze the near-optimal strategy by considering the primal abstain game \eqref{abstaingame}.
Note that the inner max-min of \eqref{abstaingame} is very similar to the no-abstain case of Section \ref{sec:game1}. 
So we can solve it similarly by taking advantage of minimax duality, 
just as previously in Lemma \ref{lem:game1val}.

\begin{lem}
\label{lem:innergameabsval}
Using Ordering \ref{remark:abstorder} of the examples w.r.t. any fixed $\vp^a$, 
define $v_2 = \min \left\{ i \in [n] : \frac{1}{n} \sum_{j=1}^i \abs{a_j} \geq \lambda \right\}$. 
Then 
\begin{align*}
&\max_{\vg \in [-1,1]^n} \min_{\substack{ \vz \in [-1,1]^n , \\ \frac{1}{n} \vz^\top \va \geq \lambda }}\; \frac{1}{n} \sum_{i=1}^n z_i \lrp{1 - p_i^a} g_i \\
&= \frac{1}{n} \sum_{i=1}^{v_2 - 1} \lrp{1 - p_i^a} + \frac{1 - p_{v_2}^a}{\abs{a_{v_2}}} \lrp{ \lambda - \frac{1}{n} \sum_{i=1}^{v_2 - 1} \abs{a_i} }
\end{align*}
\end{lem}

Substituting Lemma \ref{lem:innergameabsval} into \eqref{abstaingame} 
still leaves a minimization over $\vp^a$ in \eqref{abstaingame}. 
Solving this minimization
would then lead to the minimax optimal strategy $(\vg^*, \vp^{a*})$ for the predictor. 

We are unable to solve this minimization in closed form, 
because the choice of $\vp^a$ and Ordering \ref{remark:abstorder} depend on each other. 
However, we prove a useful property of the optimal solution here.

\begin{lem}
\label{lem:gameabssolnprop}
Suppose $\vp^{a*}$ is the minimax optimal predictor's abstain strategy. 
Using Ordering \ref{remark:abstorder} of the examples w.r.t. $\vp^{a*}$, 
define $v_2$ as in Lemma \ref{lem:innergameabsval}.
If $\alpha \geq \frac{1}{2}$, then $\vp^{a*} = \vzero$.
If $\alpha < \frac{1}{2}$, then for any $i > v_2$, the minimax optimal $p_i^{a*}$ must be set so that 
$$\frac{\abs{a_{v_2} }}{1 - p_{v_2}^{a*}} = \frac{\abs{a_i }}{1 - p_{i}^{a*}}$$
\end{lem}

The near-optimal abstain rule we choose has the properties outlined by Lemma \ref{lem:gameabssolnprop}, 
but uses Ordering \ref{remark:regorder} of the examples, 
as the results of Theorem \ref{thm:abstgameexact} use this ordering.

A convenient consequence is that when $\vp^{a,alg}$ is played, 
Orderings \ref{remark:regorder} and \ref{remark:abstorder} of the examples are effectively the same.

%Finally, the abstention problem is a multi-criterion optimization 
%minimizing abstain and mistake probabilities. 
%There is a growing body of older \cite{C70} and recent \cite{EYW11, EYW12} work 
%on the resulting Pareto frontier (``risk-coverage tradeoff") and its extreme points, 
%but we do not attempt to characterize this tradeoff with our approach.

%However, if $\vp^a$ is constrained to the set $\{ \vr^\top \vp^a \leq \beta \}$ for some $\beta > 0$,
%\footnote{Also assume hereafter that $\beta < \vr^\top \vp^{a*} = \sum_{i > v} r_i \lrp{1 - \frac{\abs{a_i}}{\abs{a_v}}}$ 
%so that the coverage constraint is not vacuous.}
%then the minimax loss \eqref{outerpaval} must be re-analyzed with the extra coverage constraint.
%This can be done in the same manner as the proof of Theorem \ref{thm:gameabsapprox}. 

%The result, neglecting quantization effects, is that the optimal $\vp^a$ is any vector in the set 
%$\{ \vp \in [0,1]^n : \vr^\top \vp = \beta \mbox{ and } p_i \leq p_i^{a*} \;\;\forall i \}$. 
%In other words, on the examples where $p_i^{a*}$ abstains ($i = v+1, \dots, n$), 
%the algorithm must simply abstain less often so as to satisfy the coverage constraint. 
%\emph{Any} such choice of abstain probabilities will be minimax optimal.
%
%This is another indication that 
%$\vp^{a*}$ is central to understanding the minimax structure of abstaining with an ensemble.
%\qedinpw
%\end{nremark}

\section{Conclusion}
We have presented an analysis of aggregating an ensemble for binary classification, 
using minimax worst-case techniques to formulate it as an game and 
suggest an optimal prediction strategy $\vg^*$, 
and PAC-Bayes analysis to derive statistical learning guarantees. 

The transductive setting, in which we consider predicting on many test examples at once, 
is key to our analysis in this manuscript, as it enables us to formulate intuitively appealing and nontrivial strategies 
$\vz^*, \vg^*$ for the game without further assumptions, by studying how the ensemble errors are allocated among test 
examples. We aim to explore such arguments further in future work.

%\subsubsection*{Acknowledgments}

\newpage
\bibliography{gameConfaistats}{}
\bibliographystyle{unsrt}

\appendix
\newpage

\section{Proofs}
\label{sec:proofs}

\begin{proof}[Proof of Lemma \ref{lem:game1val}]
It suffices to find the value of the dual game \eqref{game1dual}.
Consider the inner optimization problem faced by the predictor in this game
(where $\vz$ is fixed and known to it):
\begin{align}
\label{game1innerdual}
\mbox{Find: } \displaystyle \max_{\vg} \;\; \frac{1}{n} \vz^\top \vg \;\;\;\;\;\;\; \mbox{Such that: } -\ones{n} \leq \vg \leq \ones{n}
\end{align} 
The contribution of the $i^{th}$ example to the payoff of \eqref{game1innerdual} is $\frac{1}{n} g_i z_i$; 
to maximize this within the constraint $-1 \leq g_i \leq 1$, 
it is clear that the predictor will set $g_i = \sgn(z_i)$. 
The predictor therefore plays $\tilde{\vg} = \sgn(\vz)$, where $\sgn(\cdot)$ is taken componentwise.

With this $\tilde{\vg}$,
the game (for the average performance constraint) reduces to
\begin{eqnarray}
\label{game1outerdual}
&\mbox{Find: } \displaystyle \min_{\vz} \; \frac{1}{n} \vz^\top \tilde{\vg} 
= \min_{\vz} \; \frac{1}{n} \sum_{i=1}^n \abs{z_i} \nonumber \\ 
&\mbox{Such that: } 
\frac{1}{n} \vz^\top \va \geq \lambda \mbox{ and } -\ones{n} \leq \vz \leq \ones{n} 
\end{eqnarray}
For any $i \in [n]$, 
changing the value of $z_i$ from $0$ to $\epsilon$ raises the payoff by $\frac{1}{n} \abs{\epsilon}$, 
and can raise $\frac{1}{n} \vz^\top \va$ by at most $\frac{1}{n} \abs{a_i \epsilon}$. 
Thus, the data examples which allow nature to progress most towards satisfying the performance constraint are those with the highest $\abs{a_i}$, 
for which nature should set $z_i = \pm 1$ to extract every advantage 
(avoid leaving slack in the hypercube constraint). 
This argument holds inductively for the first $v-1$ examples before the constraint $\frac{1}{n} \vz^\top \va \geq \lambda$ is satisfied; 
the $v^{th}$ example can have $\abs{z_i} < 1$ from boundary effects, 
and for $i > v$, nature would set $z_i = 0$ to minimize the payoff.
(All this can also be shown by checking the KKT conditions.)

Consequently, the $\vz$ that solves \eqref{game1outerdual} can be defined by a sequential greedy procedure:
\begin{enumerate}
\item 
Initialize $\vz = 0$, ``working set" of examples $S = \emptyset$.
\item
Select an $i \in \displaystyle \argmax_{j \in [n] \setminus S} \abs{a_j}$ and set $S = S \cup \{ i \}$.
\item
If $\sum_{j \in S} \frac{1}{n} \abs{a_j} < \lambda$: set $z_i = \sgn(a_i)$ and go back to step 2.
\item
Else: set $z_i = \sgn(a_i) - \frac{1}{a_i} \lrp{ \sum_{j \in S} \abs{a_j} - n \lambda}$ and terminate, returning $\vz$.
\end{enumerate}

Call the vector set by this procedure $\tilde{\vz}$. 
Then under the constraints of \eqref{game1dual}, 
$\displaystyle \min_{\vz} \max_{\vg} \; \frac{1}{n} \vz^\top \vg = \max_{\vg} \; \frac{1}{n} \tilde{\vz}^\top \vg 
= \frac{1}{n} \sum_{i=1}^n \abs{\tilde{z}_i} 
= \frac{v-1}{n} + \frac{1}{n} \abs{ \sgn(a_v) - \frac{1}{a_v} \lrp{ \sum_{i=1}^v \abs{a_i} - n \lambda}} 
= \frac{v-1}{n} + \frac{1}{\abs{a_v}} \lrp{ \lambda - \sum_{i=1}^{v-1} \frac{1}{n} \abs{a_i} } = V$
, as desired.
\end{proof}

\begin{proof}[Proof of Theorem \ref{thm:game1soln}]
We have already considered the dual game in Lemma \ref{lem:game1val}, 
from which it is clear that if $\vz^*$ is played, then regardless of $\vg$,
$\frac{1}{n} \vz^*{^\top} \vg \leq \frac{v-1}{n} + \frac{1}{\abs{a_v}} \lrp{ \lambda - \frac{1}{n} \sum_{i=1}^{v-1} \abs{a_i} } = V$, 
and therefore $\vz^*$ is minimax optimal.

Now it suffices to prove that the predictor can force a correlation of $\geq V$ by playing $\vg^*$ in the primal game, 
where it plays first. 
After $\vg^*$ is played, 
nature is faced with the following problem:
\begin{eqnarray}
\mbox{Find: }&\displaystyle \min_{\vz} \;\; \frac{1}{n} \vz^\top \vg^* \label{optplaygame1}\\
\mbox{Such that: }& \frac{1}{n} \vz^\top \va \geq \lambda \mbox{ and } -\ones{n} \leq \vz \leq \ones{n} \notag
\end{eqnarray}

Now since $\frac{1}{n} \vz^\top \va \geq \lambda$, we have the inequality
\begin{align}
\label{game1optpayoff}
\frac{1}{n} \vz^\top \vg^* &= \frac{1}{n} \sum_{i=1}^{v} \sgn(a_i) z_i + \frac{1}{n \abs{a_v}} \sum_{i=v+1}^{n} a_i z_i \nonumber \\
&\geq \frac{1}{n} \sum_{i=1}^{v} \sgn(a_i) z_i + \frac{1}{\abs{a_v}} \lrp{ \lambda - \frac{1}{n} \sum_{i=1}^{v} a_i z_i } \nonumber \\
&= \frac{1}{n} \sum_{i=1}^{v-1} \sgn(a_i) z_i \lrp{1 - \frac{\abs{a_i}}{\abs{a_v}}} + \frac{\lambda}{\abs{a_v}}
\end{align}

For $i \in [1, v-1]$, $1 - \frac{\abs{a_i}}{\abs{a_v}} \leq 0$. 
So to minimize \eqref{game1optpayoff} (and solve \eqref{optplaygame1}), 
nature sets $z_i$ so that $\sgn(a_i) z_i$ is maximized. 
For each $i \leq v-1$, 
nature can force $\sgn(a_i) z_i = 1$ by setting $z_i = \sgn(a_i)$, 
and this is the maximum possible: $\sgn(a_i) z_i \leq \abs{z_i} \leq 1$. 

From \eqref{game1optpayoff}, 
we see that the values $\{z_i \}_{i \geq v}$ are irrelevant to the payoff, 
so any setting of $\vz$ that sets $z_i = \sgn(a_i)$ for $i \leq v-1$ 
(call such a setting $\tilde{\vz}$) will solve \eqref{optplaygame1}. 
\begin{align*}
\min_{\vz} \;\; \frac{1}{n} \vz^\top \vg^* &= \frac{1}{n} \tilde{\vz}^\top \vg^*
= \frac{1}{n} \sum_{i=1}^{v-1} \lrp{1 - \frac{\abs{a_i}}{\abs{a_v}}} + \frac{\lambda}{\abs{a_v}} = V
\end{align*}
so we have argued that $\vg^*$ forces a correlation of $\geq V$, and therefore it is minimax optimal.
\end{proof}

\begin{proof}[Proof of Theorem \ref{thm:abstgameexact}]
Note that the abstain game \eqref{abstaingame} is linear in all three vectors $\vg, \vp^a, \vz$. 
All three constraint sets are convex and compact, as well, 
so the minimax theorem can be invoked (see Prop. \ref{prop:game1duality}), yielding
\begin{align}
\label{abstvaldual}
V_{abst} &= \min_{\vp^a \in [0,1]^n} \min_{\vg \in [-1,1]^n} \max_{\substack{ \vz \in [-1,1]^n , \\ \frac{1}{n} \vz^\top \va \geq \lambda }} \nonumber \\
&\qquad\frac{1}{n} \sum_{i=1}^n \left[ p_i^a \alpha + \frac{1}{2} \lrp{1 - p_i^a} \lrp{ 1 - g_i z_i } \right] \nonumber \\
&= \max_{\substack{ \vz \in [-1,1]^n , \\ \frac{1}{n} \vz^\top \va \geq \lambda }}\ \min_{\vp^a \in [0,1]^n} \min_{\vg \in [-1,1]^n} \nonumber \\
&\qquad \frac{1}{n} \sum_{i=1}^n \left[ p_i^a \alpha + \frac{1}{2} \lrp{1 - p_i^a} \lrp{ 1 - g_i z_i } \right] \nonumber \\
&= \max_{\substack{ \vz \in [-1,1]^n , \\ \frac{1}{n} \vz^\top \va \geq \lambda }}\ \min_{\vp^a \in [0,1]^n} \nonumber \\
&\qquad \frac{1}{n} \sum_{i=1}^n \left[ p_i^a \alpha + \frac{1}{2} \lrp{1 - p_i^a} \lrp{ 1 - \abs{z_i} } \right] \nonumber \\
&= \max_{\substack{ \vz \in [-1,1]^n , \\ \frac{1}{n} \vz^\top \va \geq \lambda }}\ \; 
\frac{1}{n} \sum_{i=1}^n \min \lrp{ \alpha, \frac{1}{2} \lrp{1 - \abs{z_i}} }
\end{align}
From this point, the analysis is a variation on the proof of Lemma \ref{lem:game1val} 
from \eqref{game1outerdual} onwards.

Consider nature's strategy when faced with \eqref{abstvaldual}. 
A trivial upper bound on \eqref{abstvaldual} is $\alpha$ regardless of $\vz$, 
and for any $i$ it is possible to set $z_i$ such that $\abs{z_i} \leq 1 - 2 \alpha$ 
without lowering \eqref{abstvaldual} from $\alpha$. 
So nature can first set $\vz = (1 - 2 \alpha) \sgn(\va)$, 
to progress most towards satisfying the $\frac{1}{n} \vz^\top \va \geq \lambda$ 
constraint while maintaining the value at $\alpha$. 

If $\vz = (1 - 2 \alpha) \sgn(\va)$ meets the constraint $\frac{1}{n} \vz^\top \va \geq \lambda$, i.e. 
$\frac{1 - 2 \alpha}{n} \sum_{i=1}^n \abs{a_i} \geq \lambda \implies \alpha \leq \frac{1}{2} \lrp{1 - \frac{n\lambda}{\sum_{i=1}^n \abs{a_i}}}$, 
then the value is clearly $\alpha$.

Otherwise, i.e. if $c := \lambda - \frac{1 - 2 \alpha}{n} \sum_{i=1}^n \abs{a_i} > 0$, 
then nature must start with the setting $\vz' := (1 - 2 \alpha) \sgn(\va)$ and continue raising $\abs{z_i}$ for some indices $i$ 
until the constraint $\frac{1}{n} \vz^\top \va \geq \lambda$ is met. 
$c$ can be thought of as a budget that nature must satisfy by starting from $\vz'$ 
and adjusting $\{z'_i\}_{i \in S}$ away from $0$ for some subset $S$ of the indices $[n]$.

For any $i$, raising $\abs{z'_i}$ in this way by some small $\epsilon$ raises nature's payoff by $\frac{\epsilon}{2}$, 
and lowers the remaining budget by $\epsilon \abs{a_i}$.
Therefore, to satisfy the budget with maximum payoff, 
the examples get $\abs{z_i}$ set to $1$ in descending order of $\abs{a_i}$ 
(Ordering \ref{remark:regorder}, which we therefore use for the rest of this proof)
until the remaining budget to satisfy runs out.

This occurs on the $w^{th}$ example (using Ordering \ref{remark:regorder}), 
where $w = \min \left\{ i \in [n] : \frac{2 \alpha}{n} \sum_{j=1}^i \abs{a_j} \geq c \right\}$, 
which after a little algebra is equivalent to the statement of the theorem. 

Substituting this into \eqref{abstvaldual}, we get
\begin{align}
\label{vabstexact}
V_{abst} &= \frac{1}{n} \lrp{ \sum_{i=1}^{w-1} \min \lrp{ \alpha, 0} + \frac{1}{2} \lrp{1 - \abs{z_w}} + \sum_{i=w+1}^n \alpha } \nonumber \\
&= \alpha \lrp{1 - \frac{w}{n}} + \frac{1}{2 n \abs{a_w}} \lrp{cn - 2 \alpha \sum_{j=1}^{w-1} \abs{a_j} } 
\end{align}
Now by definition of $w$, 
$\frac{2 \alpha}{n} \sum_{j=1}^{w-1} \abs{a_j} \leq c \leq \frac{2 \alpha}{n} \sum_{j=1}^w \abs{a_j}$. 
Using this to bound \eqref{vabstexact} gives the result.
\end{proof}

\begin{proof}[Proof of Lemma \ref{lem:innergameabsval}]
Just as with the non-abstaining game in Lemma \ref{lem:game1val}, 
it is clear that minimax duality applies to this game. 
Therefore, we can find the value of the dual game instead, 
which is done directly using the same reasoning as in the proof of Lemma \ref{lem:game1val}. 
\begin{align*}
\max_{\vg \in [-1,1]^n}& \min_{\substack{ \vz \in [-1,1]^n , \\ \frac{1}{n} \vz^\top \va \geq \lambda }}\; \frac{1}{n} \sum_{i=1}^n z_i \lrp{1 - p_i^a} g_i \\
&= \min_{\substack{ \vz \in [-1,1]^n , \\ \frac{1}{n} \vz^\top \va \geq \lambda }}\; \max_{\vg \in [-1,1]^n} \frac{1}{n} \sum_{i=1}^n z_i \lrp{1 - p_i^a} g_i \\
&= \min_{\substack{ \vz \in [-1,1]^n , \\ \frac{1}{n} \vz^\top \va \geq \lambda }}\; \frac{1}{n} \sum_{i=1}^n \lrp{1 - p_i^a} \abs{z_i} \\
&= \frac{1}{n} \sum_{i=1}^{v_2 - 1} \lrp{1 - p_i^a} + \frac{1 - p_{v_2}^a}{\abs{a_{v_2}}} \lrp{ \lambda - \frac{1}{n} \sum_{i=1}^{v_2 - 1} \abs{a_i} }
\end{align*}
\end{proof}

\begin{proof}[Proof of Lemma \ref{lem:gameabssolnprop}]
Define $v_2$ as in Lemma \ref{lem:innergameabsval}; 
throughout this proof, we use Ordering \ref{remark:abstorder} of the examples. 
Substituting Lemma \ref{lem:innergameabsval} into \eqref{abstaingame}, 
the value of the game is
\begin{align}
\label{outerpaval}
V_{abst} = \frac{1}{2} + \frac{1}{n} \min_{\vp^a \in [0,1]^n} \Bigg[ &\sum_{i=1}^{v_2} \lrp{ \alpha p_i^a - \frac{1}{2} } \nonumber \\ 
&+ \sum_{i=v_2+1}^n \lrp{ \alpha - \frac{1}{2}} p_i^a \Bigg]
\end{align}

It only remains to (approximately) solve the minimization in \eqref{outerpaval}, 
keeping in mind that $v_2$ depends on $p_i^a$ because of the ordering of the coordinates (Ordering \ref{remark:abstorder}). 

If $\alpha \geq \frac{1}{2}$, then regardless of $i$, 
neither sum of \eqref{outerpaval} can increase with increasing $p_i^a$. 
In this case, the minimizer $\vp^{a*} = \vzero$, identically zero for all examples. 

If $\alpha < \frac{1}{2}$, 
consider an example $z := x_i$ for some $i > v_2$. 
We prove that 
$\frac{\abs{a_{v_2} }}{1 - p_{v_2}^{a*}} = \frac{\abs{a_i }}{1 - p_{i}^{a*}}$. 
If this is not true, i.e. $\frac{\abs{a_{v_2} }}{1 - p_{v_2}^{a*}} > \frac{\abs{a_i }}{1 - p_{i}^{a*}}$, 
then $p_i^{a*}$ can be raised while keeping $z$ out of the top $v_2$ examples. 
This would decrease the second sum of \eqref{outerpaval} because $\alpha - \frac{1}{2} < 0$, 
which contradicts the assumption that $p_i^{a*}$ is optimal.
\end{proof}

\begin{proof}[Proof of Theorem \ref{thm:gameabsapprox}]
Under the optimal strategy $\vp^{a*}$, 
Orderings \ref{remark:regorder} and \ref{remark:abstorder} are identical for our purposes, so $v_2 = v$. 
Revisiting the argument of Lemma \ref{lem:innergameabsval} with this information, 
the predictor's and Nature's strategies $\vg^*$ and $\vz^*$ 
are identical to their minimax optimal strategies in the non-abstaining case. 
Substituting $\vp^{a,alg}$ into \eqref{outerpaval} gives the worst-case loss after simplification.
\end{proof}

\section{PAC-Style Results for the Abstain Game}
\label{sec:pbcalc}
This appendix contains the calculations used to prove the results of Section \ref{sec:pbabstguarantees}.

If the algorithm abstains at all ($\alpha < \frac{1}{2}$), 
the overall probability that it does so is
\begin{align*}
\frac{1}{n} \sum_{i=1}^n p_i^{a,alg} &= \frac{1}{n} \sum_{i > v} \lrp{1 - \frac{\abs{a_i}}{\abs{a_v}} } \\
&\leq 1 - \frac{1}{n} \sum_{i \leq v} \abs{a_i} - \frac{1}{n} \sum_{i > v} \frac{\abs{a_i}}{\abs{a_v}} \\
&\leq 1 - \lambda - \frac{1}{n} \sum_{i > v} \frac{\abs{a_i}}{\abs{a_v}}
\end{align*}
Using a union bound with \eqref{1phaselambda} gives the result on the abstain probability.

When $\alpha < \frac{1}{2}$, the probability that the predictor predicts an incorrect label ($\neq \perp$) under minimax optimal play 
depends on nature's play $\vz$. 
It can be calculated as (w.p. $\geq 1 - \delta$):
\begin{align}
\label{absterrcalc}
\frac{1}{2n} &\sum_{i=1}^n \lrp{1 - p_i^{a*}} \lrp{ 1 - g_i^* z_i } \\ 
&= \frac{1}{2n} \sum_{i=1}^v \lrp{ 1 - \sgn(a_i) z_i } + \frac{1}{2n} \sum_{i>v} \frac{\abs{a_i}}{\abs{a_v}} \lrp{ 1 - \frac{a_i}{\abs{a_v}} z_i } \nonumber \\
&= \frac{1}{2n} \Bigg[ \sum_{i=1}^n \min \lrp{1, \frac{\abs{a_i}}{\abs{a_v}} } - \sum_{i \leq v} z_i \sgn(a_i) - \sum_{i>v} z_i \sgn(a_i) \frac{a_i^2}{a_v^2} \Bigg] \nonumber \\
&:= p_{inc} (\vz)
\end{align}

Under the constraints $-\ones{n} \leq \vz \leq \ones{n}$ and $\frac{1}{n} \vz^\top \va \geq \lambda$, 
the maximizer of \eqref{absterrcalc} w.r.t. $\vz$ is indeed $\vz^*$, 
so the chance of predicting incorrectly is
\begin{align*}
p_{inc} (\vz) \leq p_{inc} (\vz^*) &= \frac{1}{2n} \left[ \sum_{i=1}^n \min \lrp{1, \frac{\abs{a_i}}{\abs{a_v}} } - \sum_{i \leq v} 1 \right] \\
&\leq \frac{n-v}{2n} = \frac{1}{2} \lrp{1 - \frac{v}{n} } \nonumber \\
&\leq \frac{1}{2} \lrp{1 - \lambda} - \frac{1}{2n} \sum_{i=1}^v \lrp{1 - \abs{a_i}}
\end{align*}
Using a union bound with \eqref{1phaselambda} gives the result on the probability of an incorrect non-abstention.

\section{Application of the Minimax Theorem}
\begin{prop}
\label{prop:game1duality}
Let $R = \{ \vv \in \RR^n : -\ones{n} \leq \vv \leq \ones{n} \}$ and 
$A = \left\{ \vz \in \RR^n : \frac{1}{n} \vz^\top \va \geq \lambda \right\}$. 
Then 
\begin{align*}
\max_{\vg \in R} \min_{\vz \in R \cap A} \;\; \frac{1}{n} \vz^\top \vg = \min_{\vz \in R \cap A} \max_{\vg \in R} \;\; \frac{1}{n} \vz^\top \vg
\end{align*}
\end{prop}
\begin{proof}
Both $R$ and $A$ are convex and compact, as is $R \cap A$. 
The payoff function $\frac{1}{n} \vz^\top \vg$ is linear in $\vz$ and $\vg$. 
Therefore the minimax theorem (e.g. \cite{CBL06}, Theorem 7.1) applies, giving the result.
\end{proof}

\end{document}